\documentclass{article}
\usepackage{nips13submit_e,times}
\usepackage{hyperref}
\usepackage{url}

\usepackage{amsmath}
\usepackage{amssymb}
\usepackage{amsthm}
\usepackage{graphicx}
\usepackage{subfig}
\newtheorem{theorem}{Theorem}
\newtheorem{lemma}[theorem]{Lemma}
\graphicspath{{figures/}}

\title{A Latent Source Model for \\
Nonparametric Time Series Classification}

\author{George H.~Chen \\
        MIT \\ 
        \texttt{georgehc@mit.edu}
        \And
        Stanislav Nikolov \\
        Twitter \\
        \texttt{snikolov@twitter.com}
        \And
        Devavrat Shah \\
        MIT \\ 
        \texttt{devavrat@mit.edu}
        }

\nipsfinalcopy 

\begin{document}

\maketitle

\begin{abstract}
For classifying time series, a nearest-neighbor approach is widely used in
practice with performance often competitive with or better than more elaborate
methods such as neural networks, decision trees, and support vector machines.
We develop theoretical justification for the effectiveness of
nearest-neighbor-like classification of time series. Our guiding hypothesis
is that in many applications, such as forecasting which topics will become
trends on Twitter, there aren't actually that many prototypical time series to
begin with, relative to the number of time series we have access to, e.g.,
topics become trends on Twitter only in a few distinct manners whereas we can
collect massive amounts of Twitter data. To operationalize this hypothesis, we
propose a {\em latent source model} for time series, which naturally leads to
a ``weighted majority voting'' classification rule that can be approximated by
a nearest-neighbor classifier. We establish nonasymptotic performance
guarantees of both weighted majority voting and nearest-neighbor
classification under our model accounting for how much of the time series we
observe and the model complexity. Experimental results on synthetic data show
weighted majority voting achieving the same misclassification rate as
nearest-neighbor classification while observing less of the time series. We
then use weighted majority to forecast which news topics on Twitter become
trends, where we are able to detect such ``trending topics'' in advance of
Twitter 79\% of the time, with a mean early advantage of 1 hour and 26
minutes, a true positive rate of 95\%, and a false positive rate of 4\%.

\end{abstract}


\section{Introduction}

Recent years have seen an explosion in the availability of time series data
related to virtually every human endeavor --- data that demands to be analyzed
and turned into valuable insights. A key recurring task in mining this data is
being able to classify a time series. As a running example used throughout
this paper, consider a time series that tracks how much activity there is for
a particular news topic on Twitter. Given this time series up to present time,
we ask ``will this news topic go viral?'' Borrowing Twitter's terminology, we
label the time series a ``trend'' and call its corresponding news topic a
{\em trending topic} if the news topic goes viral; otherwise, the time series
has label ``not trend''. We seek to forecast whether a news topic will become
a trend {\it before} it is declared a trend (or not) by Twitter, amounting to
a binary classification problem. Importantly, we skirt the discussion of what
makes a topic considered trending as this is irrelevant to our mathematical
development.%
\footnote{While it is not public knowledge how Twitter defines a topic to be a
trending topic, Twitter does provide information for which topics are trending
topics. We take these labels to be ground truth, effectively treating how a
topic goes viral to be a black box supplied by Twitter.}
Furthermore, we remark that handling the case where a single time series can
have different labels at different times is beyond the scope of this paper.

Numerous standard classification methods have been tailored to classify time
series, yet a simple nearest-neighbor approach is hard to beat in terms of
classification performance on a variety of datasets \cite{xi_2006}, with
results competitive to or better than various other more elaborate methods
such as neural networks \cite{nanopoulos_2001}, decision trees
\cite{rodriguez_2004}, and support vector machines \cite{wu_2004}. More
recently, researchers have examined which distance to use with
nearest-neighbor classification 
\cite{gustavo_2011,ding_2008,weinberger_2009}
or how to boost classification performance by applying different
transformations to the time series before using nearest-neighbor
classification \cite{bagnall_2012}. These existing results are mostly
experimental, lacking theoretical justification for both when
nearest-neighbor-like time series classifiers should be expected to perform
well and how well.


If we don't confine ourselves to classifying time series, then as the amount
of data tends to infinity, nearest-neighbor classification has been shown to
achieve a probability of error that is at worst twice the Bayes error rate,
and when considering the nearest $k$ neighbors with $k$ allowed to grow with
the amount of data, then the error rate approaches the Bayes error rate
\cite{cover_1967}. However, rather than examining the asymptotic case where
the amount of data goes to infinity, we instead pursue {\em nonasymptotic}
performance guarantees in terms of how large of a training dataset we have and
how much we observe of the time series to be classified. To arrive at these
nonasymptotic guarantees, we impose a low-complexity structure on time series.


\textbf{Our contributions.} We present a model for which nearest-neighbor-like
classification performs well by operationalizing the following hypothesis: In
many time series applications, there are only a small number of prototypical
time series relative to the number of time series we can collect. For example,
posts on Twitter are generated by humans, who are often behaviorally
predictable in aggregate. This suggests that topics they post about only
become trends on Twitter in a few distinct manners, yet we have at our
disposal enormous volumes of Twitter data. In this context, we present a novel
{\em latent source model}: time series are generated from a small collection
of $m$ unknown latent sources, each having one of two labels, say ``trend'' or
``not trend''.
Our model's maximum a posteriori (MAP) time series classifier can be
approximated by weighted majority voting, which compares the time series to be
classified with each of the time series in the labeled training data. Each
training time series casts a weighted vote in favor of its ground truth label,
with the weight depending on how similar the time series being classified is
to the training example. The final classification is ``trend'' or ``not
trend'' depending on which label has the higher overall vote. The voting is
nonparametric in that it does not learn parameters for a model and is driven
entirely by the training data. The unknown latent sources are never estimated;
the training data serve as a proxy for these latent sources. Weighted majority
voting itself can be approximated by a nearest-neighbor classifier, which we
also analyze.


Under our model, we show sufficient conditions so that if we have
$n=\Theta(m \log \frac{m}{\delta})$ time series in our training data, then
weighted majority voting and nearest-neighbor classification correctly
classify a new time series with probability at least $1-\delta$ after
observing its first $\Omega(\log \frac{m}{\delta})$ time steps. As our
analysis accounts for how much of the time series we observe, our results
readily apply to the ``online'' setting in which a time series is to be
classified while it streams in (as is the case for forecasting trending
topics) as well as the ``offline'' setting where we have access to the entire
time series. Also, while our analysis yields matching error upper bounds for
the two classifiers, experimental results on synthetic data suggests that
weighted majority voting outperforms nearest-neighbor classification early on
when we observe very little of the time series to be classified. Meanwhile,
a specific instantiation of our model leads to a spherical Gaussian mixture
model, where the latent sources are Gaussian mixture components. We show that
existing performance guarantees on learning spherical Gaussian mixture models
\cite{dasgupta_2007,hsu_2013,vempala_wang}
require more stringent conditions
than what our results need, suggesting that learning the latent
sources is overkill if the goal is classification.

Lastly, we apply weighted majority voting to forecasting trending topics on
Twitter. We emphasize that our goal is {\em precognition} of trends:
predicting whether a topic is going to be a trend before it is actually
declared to be a trend by Twitter or, in theory, any other third party that we
can collect ground truth labels from. Existing work that identify trends on
Twitter \cite{Becker, Cataldi, Mathioudakis} instead, as part of their trend
detection, define models for what trends are, which we do not do, nor do we
assume we have access to such definitions. (The same could be said of previous
work on novel document detection on Twitter
\cite{kasiviswanathan_2011,kasiviswanathan_2012}.) In our experiments,
weighted majority voting is able to predict whether a topic will be a trend in
advance of Twitter 79\% of the time, with a mean early advantage of 1 hour and
26 minutes, a true positive rate of 95\%, and a false positive rate of~4\%.
We empirically find that the
Twitter activity of a news topic that becomes a trend tends to follow one of a
finite number of patterns, which could be thought of as latent sources. 

\textbf{Outline.} Weighted majority voting and nearest-neighbor classification
for time series are presented in Section~\ref{sec:problem}. We provide our
latent source model and theoretical performance guarantees of weighted
majority voting and nearest-neighbor classification under this model in
Section~\ref{sec:theory}.
Experimental results for synthetic data and forecasting trending topics on
Twitter are in
Section~\ref{sec:experiments}. 


\section{Weighted Majority Voting and Nearest-Neighbor Classification}
\label{sec:problem}

Given a time-series%
\footnote{We index time using $\mathbb{Z}$ for notationally convenience but
will assume time series to start at time step 1.} 
$s:\mathbb{Z}\rightarrow\mathbb{R}$, we want to classify it as having either
label $+1$ (``trend'') or $-1$ (``not trend''). To do so, we have access to
labeled training data $\mathcal{R}_+$ and $\mathcal{R}_-$, which denote the
sets of all training time series with labels $+1$ and $-1$ respectively.

\textbf{Weighted majority voting.}
Each positively-labeled example $r\in\mathcal{R}_+$ casts a weighted vote
$e^{-\gamma d^{(T)}(r,s)}$ for whether time series $s$ has label $+1$, where
$d^{(T)}(r,s)$ is some measure of similarity between the two time series $r$
and $s$, superscript $(T)$ indicates that we are only allowed to look at the
first $T$ time steps (i.e., time steps $1,2,\dots,T$) of $s$ (but we're
allowed to look outside of these time steps for the training time series $r$),
and constant $\gamma\ge0$ is a scaling parameter that determines the ``sphere
of influence'' of each example. Similarly, each negatively-labeled example in
$\mathcal{R}_-$ also casts a weighted vote for whether time series $s$ has
label $-1$. 

The similarity measure $d^{(T)}(r,s)$ could, for example, be squared
Euclidean distance:
$d^{(T)}(r,s)=\sum_{t=1}^T (r(t)-s(t))^2\triangleq\|r - s\|_T^2$.
However, this similarity measure only looks at the first $T$ time steps of
training time series $r$.
Since time series in our training data are known, we need not
restrict our attention to their first $T$ time steps.
Thus, we use the following similarity measure:
\begin{equation}
d^{(T)}(r,s)
=
\min_{\Delta\in\{-\Delta_{\max},\dots,0,\dots,\Delta_{\max}\}}
  \sum_{t=1}^T (r(t+\Delta)-s(t))^2
=
\min_{\Delta\in\{-\Delta_{\max},\dots,0,\dots,\Delta_{\max}\}}
  \|r*\Delta - s\|_T^2,
\end{equation}
where we minimize over integer time shifts with a pre-specified maximum
allowed shift $\Delta_{\max}\ge0$. Here, we have 
used $q * \Delta$ to denote time series $q$ advanced by $\Delta$ time steps,
i.e., $(q*\Delta)(t)=q(t+\Delta)$. 

Finally, we sum up all of the weighted $+1$ votes and then all of the weighted
$-1$ votes. The label with the majority of overall weighted votes is declared
as the label for $s$:
\begin{equation}
\widehat{L}^{(T)}(s;\gamma)
=\begin{cases}
 +1 & \text{if }
   \sum_{r\in\mathcal{R}_+}e^{-\gamma d^{(T)}(r,s)} \ge
   \sum_{r\in\mathcal{R}_-}e^{-\gamma d^{(T)}(r,s)}, \\
 -1 & \text{otherwise}.
\end{cases}
\label{eq:decision-rule-with-min}
\end{equation}
Using a larger time window size $T$ corresponds to waiting longer before we
make a prediction. We need to trade off how long we wait and how accurate we
want our prediction.
Note that $k$-nearest-neighbor classification
corresponds to only considering the $k$ nearest
neighbors of $s$ among all training time series; all other votes are set to 0.
With $k=1$, we obtain the
following classifier:

\textbf{Nearest-neighbor classifier.}
Let
$\widehat{r}
 =\arg\min_{r\in\mathcal{R}_+\cup\mathcal{R}_-} d^{(T)}(r,s)$
be the nearest neighbor of $s$. Then we declare the label for $s$ to be:
\begin{equation}
\widehat{L}_{NN}^{(T)}(s)
=\begin{cases}
 +1 & \text{if }\widehat{r}\in\mathcal{R}_+, \\
 -1 & \text{if }\widehat{r}\in\mathcal{R}_-.
\end{cases}
\label{eq:decision-rule-nn}
\end{equation}

\section{A Latent Source Model and Theoretical Guarantees}
\label{sec:theory}

We assume there to be $m$ unknown latent sources (time series) that
generate observed time series. Let~$\cal V$ denote the set of all such latent
sources; each latent source $v:\mathbb{Z}\rightarrow\mathbb{R}$ in
$\mathcal{V}$ has a true label $+1$ or $-1$. Let ${\cal V}_+\subset\cal V$
be the set of latent sources with label $+1$, and ${\cal V}_-\subset\cal V$ be
the set of those with label $-1$. The observed time series are
generated from latent sources as follows:
\begin{enumerate}

\item Sample latent source $V$ from ${\cal V}$ uniformly at random.%
  \footnote{While we keep the sampling uniform for clarity of presentation,
  our theoretical guarantees can easily be extended to the case where the
  sampling is not uniform. The only change is that the number of training
  data needed will be larger by a factor of $\frac{1}{m\pi_{\min}}$, where
  $\pi_{\min}$ is the smallest probability of a particular latent source
  occurring.}
  Let $L \in \{\pm1\}$ be the label of $V$.

\item Sample integer time shift $\Delta$ uniformly from
  $\{0,1,\dots,\Delta_{\max}\}$.

\item Output time series $S: \mathbb{Z}\rightarrow\mathbb{R}$ to be latent
  source~$V$ advanced by $\Delta$ time steps, followed by adding noise signal
  $E: \mathbb{Z}\rightarrow\mathbb{R}$, i.e., $S(t) = V(t + \Delta) + E(t)$.
  The label associated with the generated time series $S$ is the same as that
  of $V$, i.e., $L$. Entries of noise $E$ are i.i.d.~zero-mean sub-Gaussian
  with parameter $\sigma$, which means that for any time index $t$,
  \begin{equation}
  \mathbb{E}[\exp(\lambda E(t))]
    \le \exp\Big(\frac{1}{2}\lambda^2\sigma^2\Big)
  \qquad
  \text{for all }\lambda\in\mathbb{R}.
  \end{equation}
  The family of sub-Gaussian distributions includes a variety of
  distributions, such as a zero-mean Gaussian with standard deviation
  $\sigma$ and a uniform distribution over $[-\sigma, \sigma]$.

\end{enumerate}

The above generative process defines our latent source model. Importantly, we
make no assumptions about the structure of the latent sources. For instance,
the latent sources could be tiled as shown in Figure~\ref{fig:interlace},
where they are evenly separated vertically and alternate between the two
different classes $+1$ and $-1$. With a parametric model like a $k$-component
Gaussian mixture model, estimating these latent sources could be problematic.
For example, if we take any two adjacent latent sources with label $+1$ and
cluster them, then this cluster could be confused with the latent source
having label $-1$ that is sandwiched in between. 
Noise only complicates estimating the latent sources.
In this example, the $k$-component
Gaussian mixture model needed for label $+1$ would require~$k$ to be the exact
number of latent sources with label $+1$, which is unknown. In general, the
number of samples we need from a Gaussian mixture mixture model to estimate
the mixture component means is exponential in the number of mixture components
\cite{moitra_2010}. As we discuss next, for classification, we 
sidestep
learning the latent sources altogether, instead using training data as a proxy
for latent sources. At the end of this section, we compare our sample
complexity for classification versus some
existing sample complexities for learning Gaussian mixture models.

\begin{figure}
\centering
\includegraphics[width=1.7in]{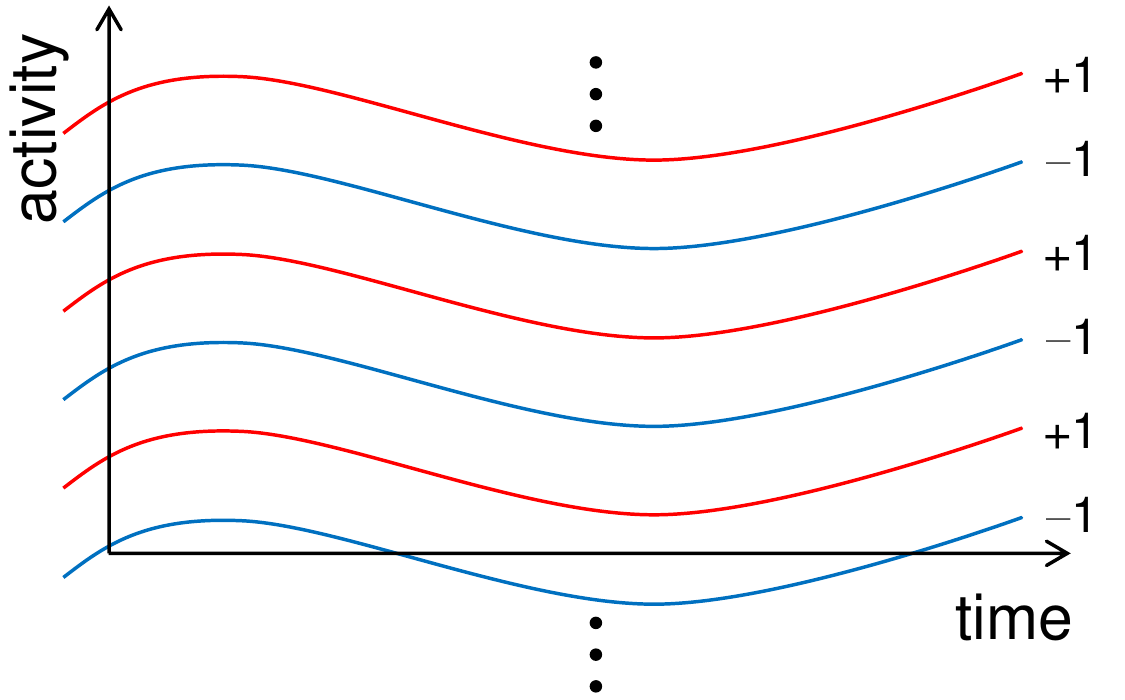}
\caption{Example of latent sources superimposed, where each latent source is
  shifted vertically in amplitude such that every other latent source has
  label $+1$ and the rest have label $-1$.}
\label{fig:interlace}
\end{figure}

\textbf{Classification.} If we knew the latent sources and if noise entries $E(t)$
were i.i.d.~$\mathcal{N}(0,\frac{1}{2\gamma})$ across $t$, then the
{\em maximum a posteriori} (MAP) estimate for label $L$ given an observed time
series $S=s$ is
\begin{equation}
\widehat{L}_{\text{MAP}}^{(T)}(s;\gamma)
=\begin{cases}
   +1 & \text{if }\Lambda_{\text{MAP}}^{(T)}(s;\gamma) \ge 1,\\
   -1 & \text{otherwise},
 \end{cases}
\label{eq:decision-rule}
\end{equation}
where
\begin{equation}
\Lambda_{\text{MAP}}^{(T)}(s;\gamma)
\triangleq
  \frac{\sum_{v_+\in\mathcal{V}_+}
        \sum_{\Delta_+\in {\cal D}_+}
        \exp\big(-\gamma \|v_+*\Delta_+ - s\|_T^2\big)}
       {\sum_{v_-\in\mathcal{V}_-}
        \sum_{\Delta_-\in {\cal D}_+}
   \exp\big( -\gamma \|v_-*\Delta_- - s\|^2_T \big)},
\label{eq:likelihood-ratio-MAP}
\end{equation}
and ${\cal D}_+  \triangleq \{0,\dots, \Delta_{\text{max}}\}$.

However, we do not know the latent sources, nor do we know if the noise is
i.i.d.~Gaussian. We assume that we have access to training data as given in
Section~\ref{sec:problem}. We make a further assumption that the training data
were sampled from the latent source model and that we have $n$ different
training time series. 
Denote
${\cal D} \triangleq
 \{-\Delta_{\text{max}}, \dots, 0,\dots, \Delta_{\text{max}}\}$.
Then we approximate the MAP classifier by using training data as a proxy for
the latent sources. Specifically, we take
ratio~\eqref{eq:likelihood-ratio-MAP}, replace the inner sum by a minimum in
the exponent, replace $\mathcal{V}_+$ and $\mathcal{V}_-$ by
$\mathcal{R}_+$ and $\mathcal{R}_-$, and replace $\mathcal{D}_+$ by
$\mathcal{D}$ to obtain the ratio:
\begin{equation}
\label{eq:est}
\Lambda^{(T)}(s;\gamma)
\triangleq
  \frac{\sum_{r_+\in\mathcal{R}_+}
          \exp\big(-\gamma
                    \big(
                      \min_{\Delta_+\in {\cal D}} \|r_+*\Delta_+ - s\|_T^2
                    \big)
              \big)}
       {\sum_{r_-\in\mathcal{R}_-}
          \exp\big(-\gamma
                    \big(
                      \min_{\Delta_-\in {\cal D}} \|r_-*\Delta_- - s\|^2_T
                    \big)
              \big)}.
\end{equation}
Plugging $\Lambda^{(T)}$ in place of $\Lambda_{\text{MAP}}^{(T)}$ in
classification rule~\eqref{eq:decision-rule} yields the weighted majority
voting rule~\eqref{eq:decision-rule-with-min}. Note that weighted majority
voting could be interpreted as a {\em smoothed} nearest-neighbor approximation
whereby we only consider the time-shifted version of each example time series
that is closest to the observed time series $s$. If we didn't replace the
summations over time shifts with minimums in the exponent, then we have a
kernel density estimate in the numerator and in the denominator
\cite[Chapter~7]{fukunaga_1990} (where the kernel is Gaussian) and our main
theoretical result for weighted majority voting to follow would still hold
using the same proof.\footnote{We use a minimum rather a summation over time
shifts to make the method more similar to existing time series classification
work (e.g., \cite{xi_2006}), which minimize over time warpings rather than
simple shifts.}

Lastly, applications may call for trading off true and false positive rates.
We can do this by generalizing decision rule \eqref{eq:decision-rule} to
declare the label of $s$ to be $+1$ if $\Lambda^{(T)}(s,\gamma)\ge\theta$ and
vary parameter $\theta>0$. The resulting decision rule, which we refer to
as {\em generalized weighted majority voting}, is thus:
\begin{equation}
\widehat{L}_\theta^{(T)}(s;\gamma)
=\begin{cases}
 +1 & \text{if }\Lambda^{(T)}(s,\gamma)\ge\theta, \\
 -1 & \text{otherwise},
\end{cases}
\label{eq:decision-rule-general}
\end{equation}
where setting $\theta=1$ recovers the usual weighted majority voting
\eqref{eq:decision-rule-with-min}. This modification to the classifier can be
thought of as adjusting the priors on the relative sizes of the two classes.
Our theoretical results to follow actually cover this more general case rather
than only that of $\theta=1$.

\textbf{Theoretical guarantees.} We now present the main theoretical results
of this paper which identify sufficient conditions under which generalized
weighted majority voting \eqref{eq:decision-rule-general} and nearest-neighbor
classification~\eqref{eq:decision-rule-nn} can classify a time series
correctly with high probability, accounting for the size of the training
dataset and how much we observe of the time series to be classified.
First, we define the ``gap'' between ${\cal R}_+$
and ${\cal R}_-$ restricted to time length $T$ and with maximum time shift
$\Delta_{\max}$ as:
\begin{equation}\label{eq:gap}
G^{(T)}(\mathcal{R}_+,\mathcal{R}_-,\Delta_{\max})
\triangleq
  \min_{\substack{r_{+}\in\mathcal{R}_+, r_{-}\in\mathcal{R}_-, \\
                  \Delta_{+}, \Delta_{-} \in {\cal D}}}
    \|r_+ * \Delta_+ - r_{-} * \Delta_{-}\|^2_T.
\end{equation}
This quantity measures how far apart the two different classes are if we only
look at length-$T$ chunks of each time series and allow all shifts of at most
$\Delta_{\max}$ time steps in either direction.

Our first main result is stated below. We defer proofs
for this section to Appendices \ref{sec:proof} and \ref{sec:proof-nn}.

\begin{theorem}
\label{thm:wmv-main-result}
(Performance guarantee for generalized weighted majority voting)
Let $m_+=|{\cal V}_+|$ be the number of latent sources with label $+1$, and
$m_-=|{\cal V}_-|=m-m_+$ be the number of latent sources with label $-1$.
For any $\beta>1$, under the latent source model with $n > \beta m \log m$
time series in the training data, the probability of misclassifying time
series $S$ with label $L$ using generalized weighted majority voting
$\widehat{L}_\theta^{(T)}(\cdot;\gamma)$ satisfies the bound
\begin{align}
&\mathbb{P}(\widehat{L}_\theta^{(T)}(S;\gamma)\ne L) \nonumber \\
&\le
    \Big( \frac{\theta m_+}{m} + \frac{m_-}{\theta m} \Big)
    (2\Delta_{\max}+1) n
    \exp\big(-(\gamma-4\sigma^2\gamma^2)
              G^{(T)}(\mathcal{R}_+,\mathcal{R}_-,\Delta_{\max})
        \big) + m^{-\beta + 1}.
\label{eq:wmv-main-bound}
\end{align}
\end{theorem}

An immediate consequence is that given error tolerance $\delta \in (0,1)$ and
with choice $\gamma \in (0, \frac{1}{4\sigma^2})$, then upper bound
\eqref{eq:wmv-main-bound} is at most $\delta$ (by having each
of the two terms on the right-hand side be $\le\frac{\delta}{2}$) if
$n> m\log \frac{2m}{\delta}$ (i.e., $\beta=1+\log\frac{2}{\delta}/\log{m}$),
and
\begin{equation}
G^{(T)}(\mathcal{R}_+,\mathcal{R}_-,\Delta_{\max})
\ge
\frac{
  \log ( \frac{\theta m_+}{m} + \frac{m_-}{\theta m} )
  + \log (2\Delta_{\max}+1)
  + \log n
  + \log \frac{2}{\delta}
}
{
  \gamma-4\sigma^2\gamma^2
}.
\end{equation}
This means that if we have access to a large enough pool of labeled time
series, i.e., the pool has $\Omega(m \log\frac{m}{\delta})$ time series, then
we can subsample $n=\Theta(m \log\frac{m}{\delta})$ of them to use as training
data. Then with choice $\gamma=\frac{1}{8\sigma^2}$, generalized weighted
majority voting \eqref{eq:decision-rule-general} correctly classifies a new
time series $S$ with probability at least $1-\delta$ if%
\begin{equation}
G^{(T)}(\mathcal{R}_+,\mathcal{R}_-,\Delta_{\max})
= \Omega\bigg(
          \sigma^2
          \Big(
          \log \Big( \frac{\theta m_+}{m} + \frac{m_-}{\theta m} \Big)
          + \log (2\Delta_{\max}+1)
          + \log \frac{m}{\delta}
          \Big)
        \bigg).
\end{equation}

Thus, the gap between sets ${\cal R}_+$ and ${\cal R}_-$ needs to grow
logarithmic in the number of latent sources $m$ in order for weighted majority
voting to classify correctly with high probability. Assuming that the original
unknown latent sources are separated (otherwise, there is no hope to
distinguish between the classes using any classifier) and the gap in the
training data grows as
$G^{(T)}(\mathcal{R}_+,\mathcal{R}_-,\Delta_{\max})=\Omega(\sigma^2 T)$
(otherwise, the closest two training time series from opposite classes are
within noise of each other), then observing the first
$T=\Omega(\log (\theta+\frac{1}{\theta})
          + \log (2\Delta_{\max}+1)
          + \log \frac{m}{\delta})$
time steps from the time series is sufficient to classify it correctly with
probability at least $1-\delta$.

A similar result holds for the nearest-neighbor classifier
\eqref{eq:decision-rule-nn}.

\begin{theorem}
\label{thm:nn-main-result}
(Performance guarantee for nearest-neighbor classification)
For any $\beta>1$, under the latent source model with $n > \beta m \log m$
time series in the training data, the probability of misclassifying time
series $S$ with label $L$ using the nearest-neighbor classifier
$\widehat{L}_{NN}^{(T)}(\cdot)$ satisfies the bound
\begin{equation}
\mathbb{P}(\widehat{L}_{NN}^{(T)}(S)\ne L)
\le
    (2\Delta_{\max}+1)
    n
    \exp\Big(-\frac{1}{16\sigma^2}
              G^{(T)}(\mathcal{R}_+,\mathcal{R}_-,\Delta_{\max})
        \Big)
    + m^{-\beta + 1}.
\label{eq:nn-main-bound}
\end{equation}
\end{theorem}

Our generalized weighted majority voting bound
\eqref{eq:wmv-main-bound} with $\theta=1$ (corresponding to regular weighted
majority voting) and $\gamma=\frac{1}{8\sigma^2}$ matches our nearest-neighbor
classification bound, suggesting that the two methods have similar behavior
when the gap grows with $T$.
%
%
In practice, we find weighted majority voting to outperform nearest-neighbor
classification when $T$ is small, and then as $T$ grows large, the two methods
exhibit similar performance in agreement with our theoretical analysis. For
small $T$, it could still be fairly likely that the nearest neighbor found has
the wrong label, dooming the nearest-neighbor classifier to failure. Weighted
majority voting, on the other hand, can recover from this situation as there
may be enough correctly labeled training time series close by that contribute
to a higher overall vote for the correct class. This robustness of weighted
majority voting makes it favorable in the online setting where we want to make
a prediction as early as possible. 

{\bf Sample complexity of learning the latent sources.}
If we can estimate the latent sources accurately, then we could plug these
estimates in place of the true latent sources in the MAP classifier and
achieve classification performance close to optimal. If we restrict the noise
to be Gaussian and assume $\Delta_{\max}=0$, then the latent source model
corresponds to a spherical Gaussian mixture model. We could learn such a model
using Dasgupta and Schulman's modified EM algorithm \cite{dasgupta_2007}.
Their theoretical guarantee depends on the true separation between the closest
two latent sources, namely
$G^{(T)*} \triangleq \min_{v,v'\in{\cal V}\text{ s.t.~}v\ne v'} \|v-v'\|_2^2$,
which needs to satisfy $G^{(T)*} \gg \sigma^2 \sqrt{T}$. Then with number of
training time series
$n=\Omega( \max\{1, \frac{\sigma^2 T}{G^{(T)*}}\} m\log\frac{m}{\delta} )$,
gap $G^{(T)*} = \Omega( \sigma^2 \log\frac{m}{\varepsilon} )$, and number of
initial time steps observed
\begin{equation}
T
=
\Omega\bigg(
        \max\bigg\{ 1,
                    \frac{\sigma^{4}T^{2}}
                         {(G^{(T)*})^2}
            \bigg\}
        \log
        \bigg[\frac{m}{\delta}
              \max\bigg\{ 1,
                          \frac{\sigma^{4}T^{2}}
                               {(G^{(T)*})^2}
                  \bigg\}
        \bigg]
      \bigg),
\end{equation}
their algorithm achieves, with probability at least $1-\delta$, an additive
$\varepsilon\sigma\sqrt{T}$ error (in Euclidean distance) close to optimal in
estimating every latent source. In contrast, our result is in terms of gap
$G^{(T)}(\mathcal{R}_+,\mathcal{R}_-,\Delta_{\max})$ that depends not on the
true separation between two latent sources but instead on the minimum observed
separation in the training data between two time series of opposite labels. In
fact, our gap, in their setting, grows as $\Omega(\sigma^2 T)$ even when their
gap $G^{(T)*}$ grows sublinear in $T$. In particular, while their result
cannot handle the regime where $O(\sigma^2\log\frac{m}{\delta})
\le G^{(T)*} \le \sigma^2 \sqrt{T}$, ours can, using
$n=\Theta(m\log\frac{m}{\delta})$ training time series and observing the first
$T=\Omega(\log\frac{m}{\delta})$ time steps to classify a time series
correctly with probability at least $1-\delta$; see Appendix
\ref{sec:gaussian-classification} for details.

Vempala and Wang \cite{vempala_wang} have a spectral method for learning
Gaussian mixture models that can handle smaller $G^{(T)*}$ than Dasgupta and
Schulman's approach but requires 
$n=\widetilde{\Omega}(T^3 m^2)$ training data, where we've hidden the
dependence on $\sigma^2$ and other variables of interest for clarity of
presentation. Hsu and Kakade \cite{hsu_2013} have a moment-based estimator that
doesn't have a gap condition but, under a different non-degeneracy condition,
requires substantially more samples for our problem setup, i.e.,
$n=\Omega((m^{14}+Tm^{11})/\varepsilon^2)$ to
achieve an $\varepsilon$ approximation of the mixture components. 
These results need substantially more training data than
what we've shown is sufficient for classification.

To fit a Gaussian mixture model to massive training datasets, in practice,
using all the training data could be prohibitively expensive. In such
scenarios, one could instead non-uniformly subsample
$\mathcal{O}(Tm^3/\varepsilon^2)$ time series from the training data using the
procedure given in~\cite{feldman_2011} and then feed the resulting smaller
dataset, referred to as an $(m,\varepsilon)$-{\em coreset}, to the EM
algorithm for learning the latent sources. This procedure still requires more
training time series than needed for classification and lacks a guarantee that
the estimated latent sources will be close to the true latent sources.



\section{Experimental Results}
\label{sec:experiments}


\begin{figure}[t]
\centering
\subfloat[][]{
\includegraphics[width=2.6in]{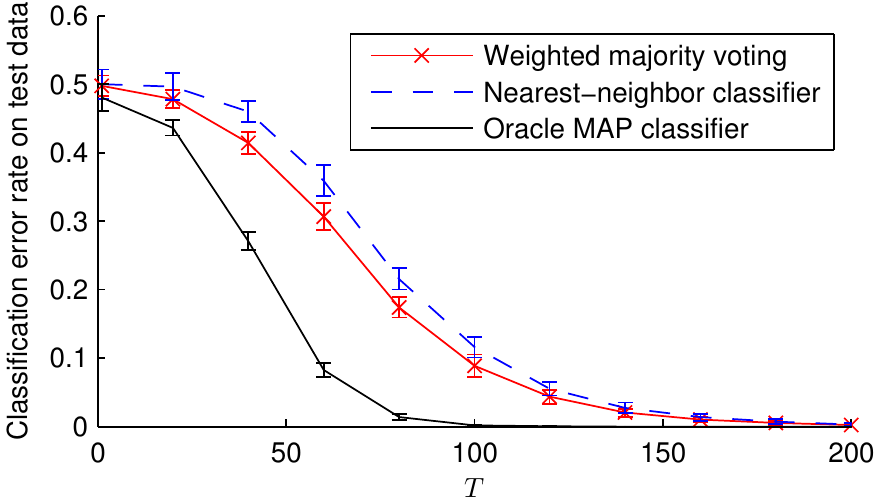}
\label{fig:error-vs-T}
}
\subfloat[][]{
\includegraphics[width=2.6in]{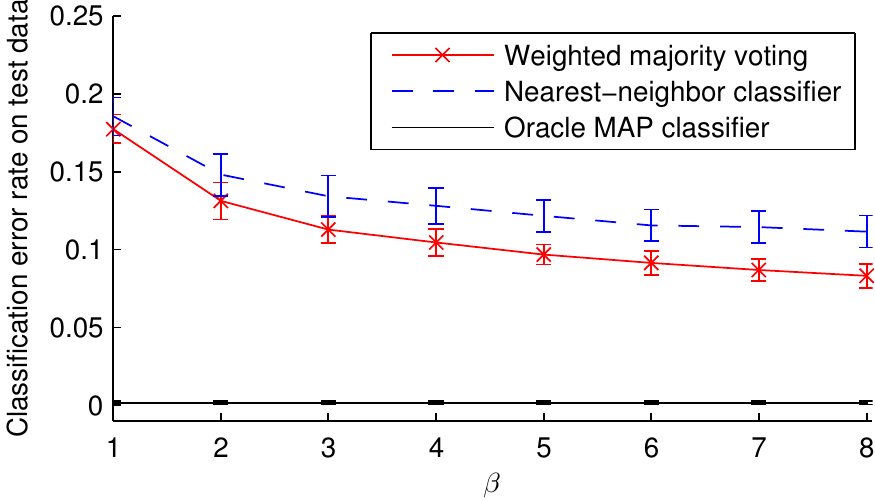}
\label{fig:error-vs-beta}
}
\caption{
         Results on synthetic data.
         (a)
         Classification error rate vs.~number of initial time steps $T$
         used; training set size: $n=\beta m\log m$ 
         where $\beta=8$.
         (b)
         Classification error rate at $T=100$ vs.~$\beta$.
         All experiments were repeated 20 times with newly generated latent
         sources, training data, and test data each time. Error bars denote
         one standard deviation above and below the mean value.}
\label{fig:synth-data}
\end{figure}
\begin{figure}[t]
\centering
\includegraphics[scale=.6, clip=true, trim=3em 1.6em 3em 2em]{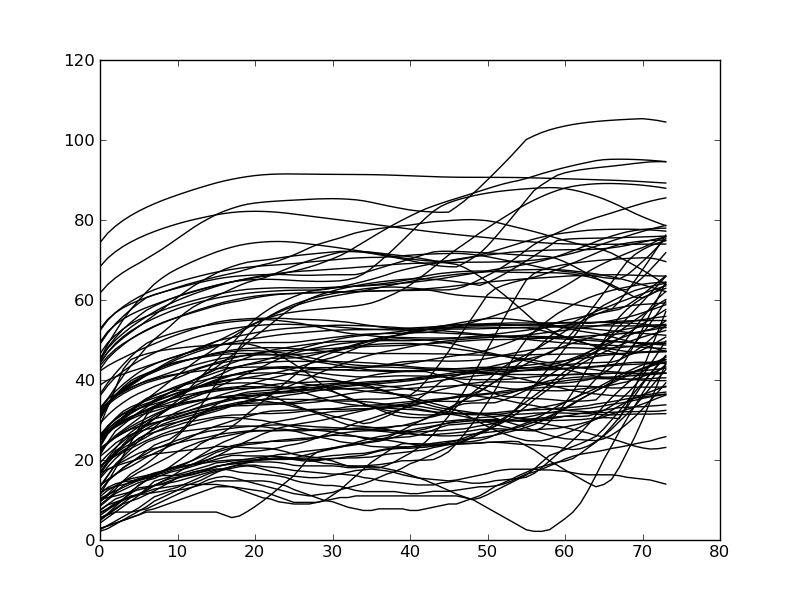}
\includegraphics[scale=.6, clip=true, trim=3em 1.6em 3em 2em]{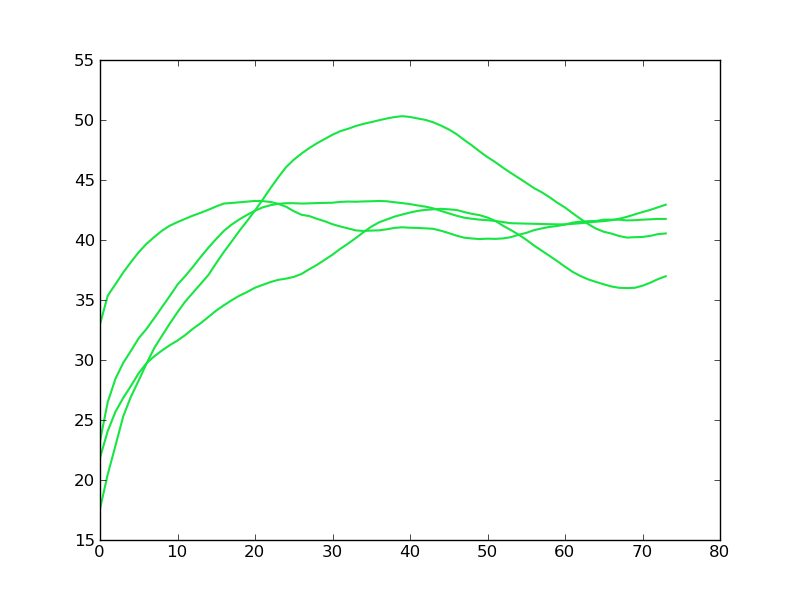}
\includegraphics[scale=.6, clip=true, trim=3em 1.6em 3em 2em]{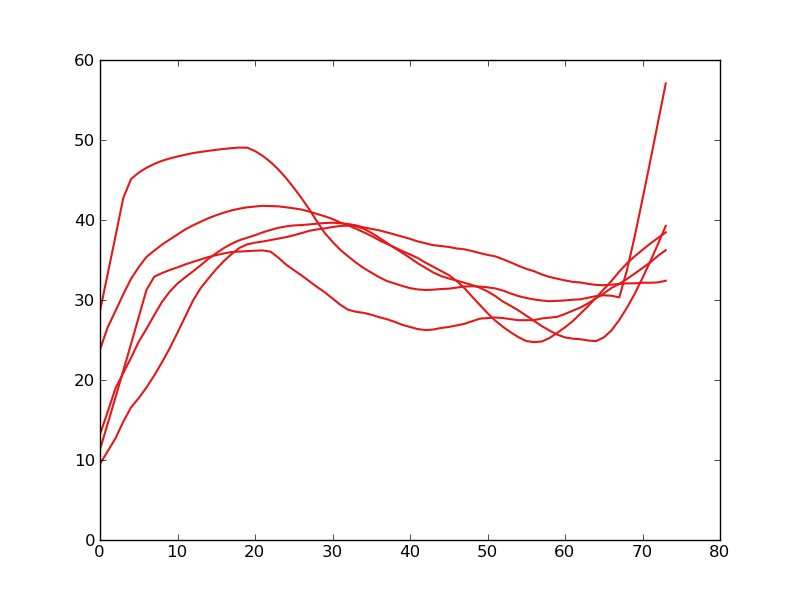}
\hspace{.5in}~~~ \\
\includegraphics[scale=.6, clip=true, trim=3em 1.6em 3em 2em]{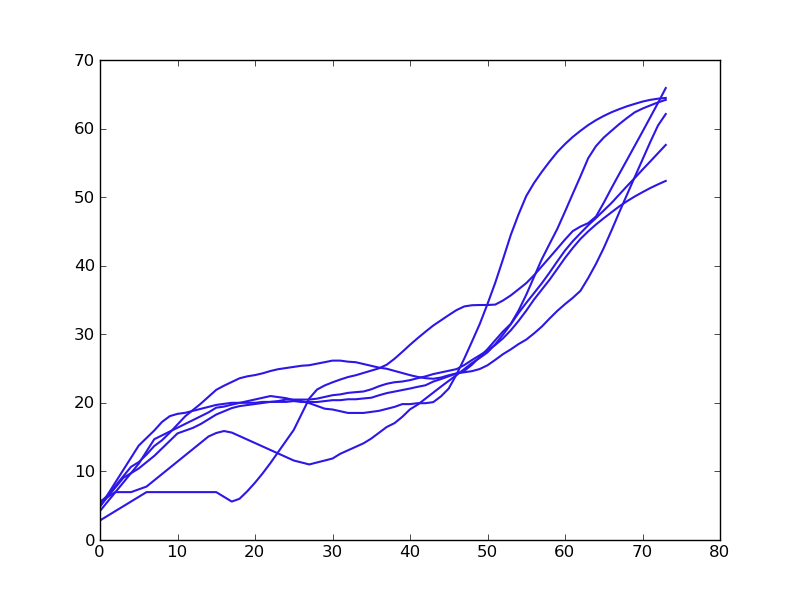}
\includegraphics[scale=.6, clip=true, trim=3em 1.6em 3em 2em]{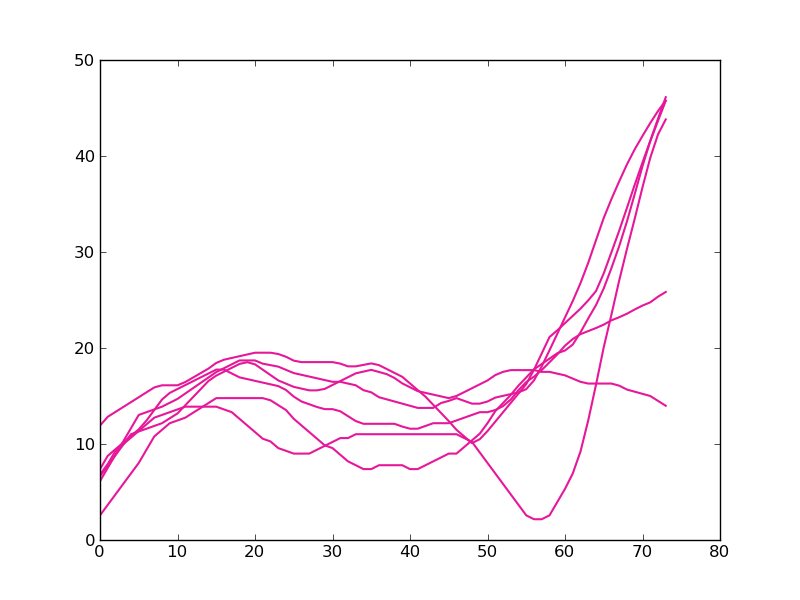}
\includegraphics[scale=.6, clip=true, trim=3em 1.6em 3em 2em]{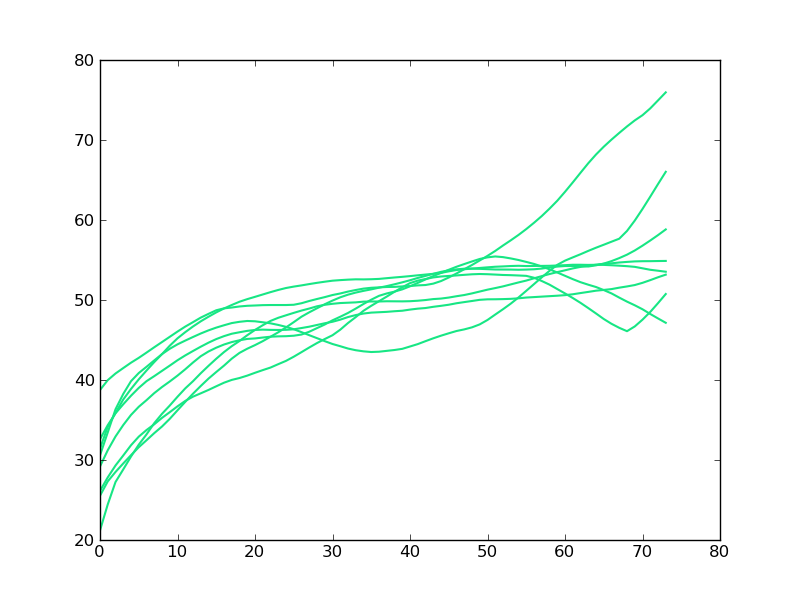}
\includegraphics[width=.6in]{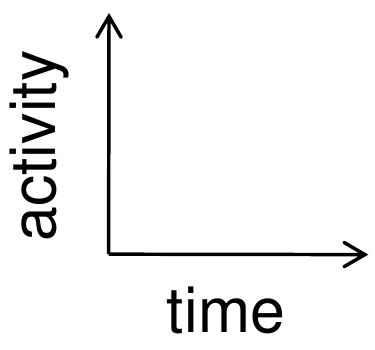}
\caption{How news topics become trends on Twitter.
The top left shows some time series of activity leading up to a news
topic becoming trending. These time series superimposed look like clutter, but
we can separate them into different clusters, as shown in the next five plots.
Each cluster represents a ``way'' that a news topic becomes trending.}
\label{fig:clusters}
\end{figure}
\textbf{Synthetic data.}
We generate $m=200$ latent sources, where each latent source is constructed by
first sampling i.i.d.~$\mathcal{N}(0,100)$ entries per time step and then
applying a 1D Gaussian smoothing filter with scale parameter 30. Half of the
latent sources are labeled $+1$ and the other half $-1$. Then
$n=\beta m\log m$ training time series are sampled as per the latent source
model where the noise added is i.i.d.~$\mathcal{N}(0,1)$ and
$\Delta_{\max}=100$. We similarly generate 1000 time series to use as test
data. We set $\gamma=1/8$ for weighted majority voting. For $\beta=8$, we
compare the classification error rates on test data for weighted majority
voting, nearest-neighbor classification, and the MAP classifier with oracle
access to the true latent sources as shown in
Figure~\ref{fig:synth-data}\subref{fig:error-vs-T}. We see that weighted majority voting outperforms
nearest-neighbor classification but as $T$ grows large, the two methods'
performances converge to that of the MAP classifier. Fixing $T=100$, we then
compare the classification error rates of the three methods using varying
amounts of training data, as shown in
Figure~\ref{fig:synth-data}\subref{fig:error-vs-beta}; the oracle MAP
classifier is also shown but does not actually depend on training data. We see
that as $\beta$ increases, both weighted majority voting and
nearest-neighbor classification steadily improve in performance.


\begin{figure}[t]
\centering
\subfloat[][]{
\includegraphics[width=3in, clip=true, trim=.1in .1in .1in .54in]{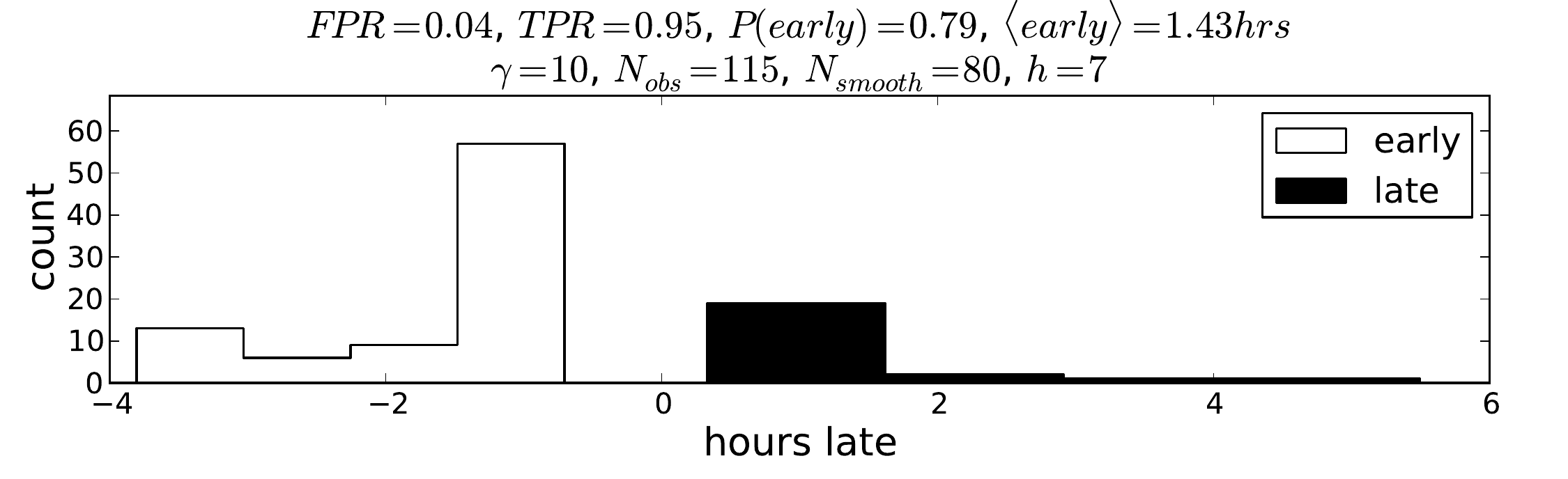}
\label{fig:early}
}~~
\subfloat[][]{
\includegraphics[height=1.1in, clip=true, trim=.1in .1in .1in .1in]{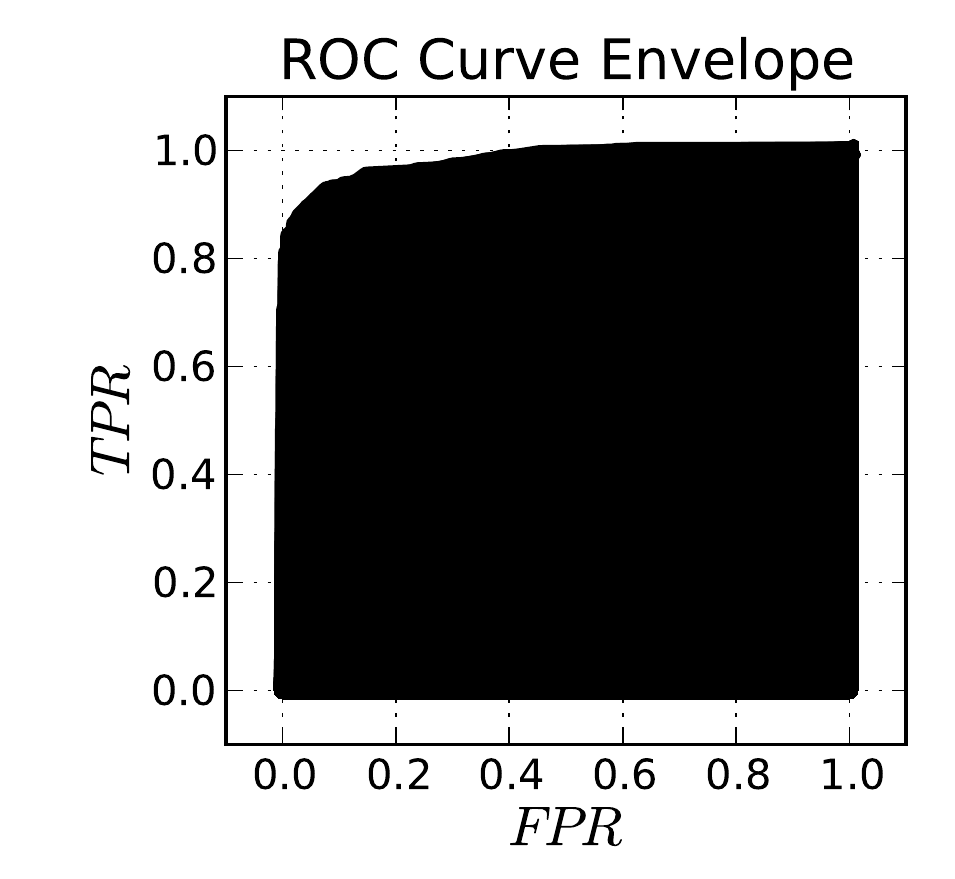}
\label{fig:roc}
}
\\
\subfloat[][]{
\includegraphics[width=2.73in, clip=true, trim=0.2in 7.4in 0.65in 0in]{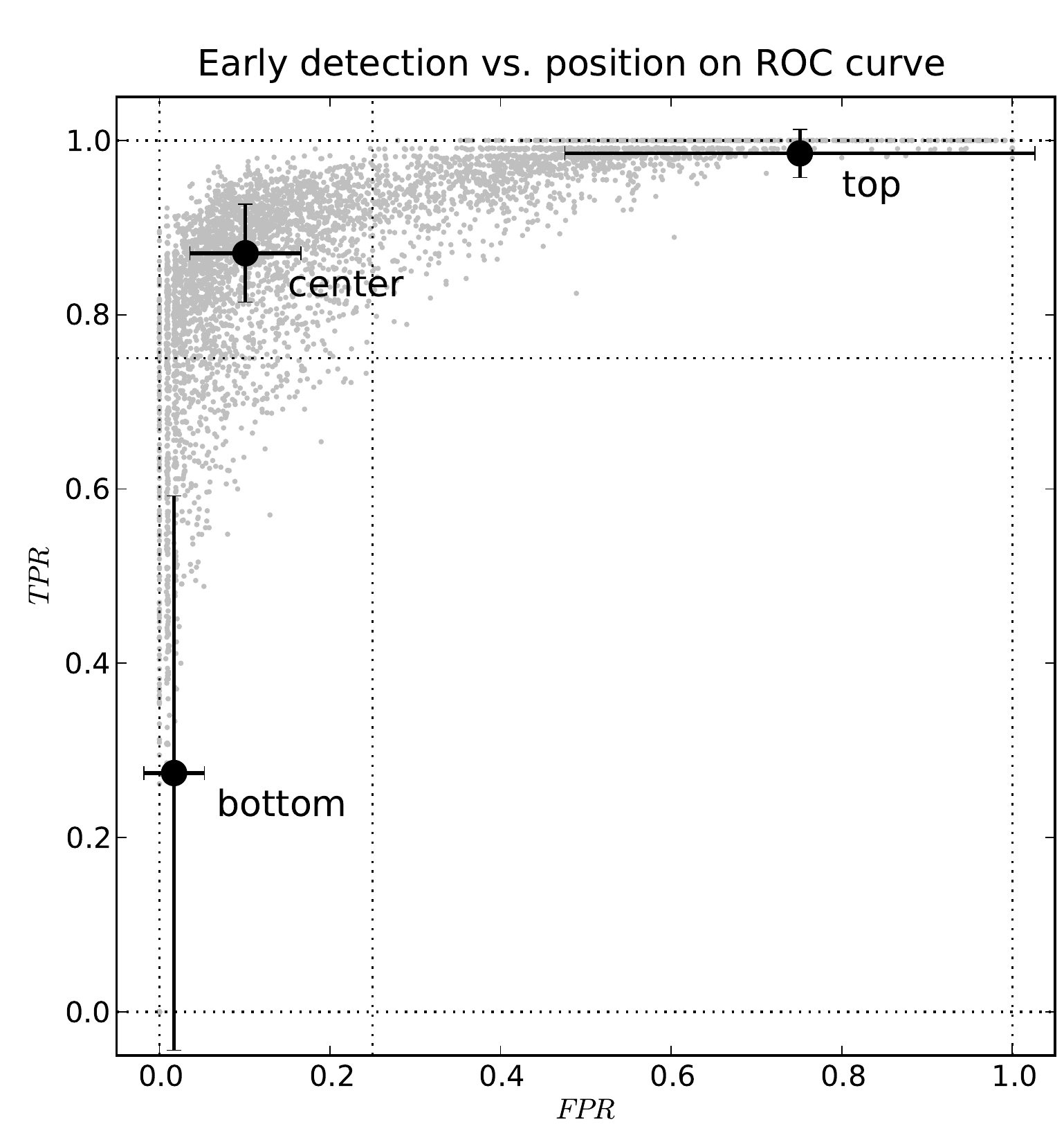}
\includegraphics[width=2.73in, clip=true, trim=0in .2in 0.5in 6.9in]{roc_early_late-eps-converted-to}
\label{fig:roc-early-late}
}
\caption{Results on Twitter data.
  (a) Weighted majority voting achieves a low error rate
  (FPR of~4\%, TPR of 95\%) and detects trending topics in advance
  of Twitter 79\% of the time, with a mean of 1.43 hours when it does;
  parameters: $\gamma=10, T=115, T_{smooth}=80, h=7$.
  (b) Envelope of all ROC curves shows the tradeoff between TPR and FPR.
  (c) Distribution of 
  detection times for ``aggressive'' (top), ``conservative'' (bottom) and
  ``in-between'' (center) parameter settings.
  \label{fig:twitter-main}}
\end{figure}


\textbf{Forecasting trending topics on twitter.}
We provide only an overview of our Twitter results here, deferring full
details to 
Appendix \ref{sec:forecasting-trends}.
We sampled 500 examples of trends at
random from a list of June 2012 news trends, and 500 examples of non-trends
based on phrases appearing in user posts during
the same month. As we do not know how Twitter chooses what phrases are
considered as candidate phrases for trending topics, it's unclear what the
size of the non-trend category is in comparison to the size of the trend
category. Thus, for simplicity, we intentionally control for the class sizes
by setting them equal. In practice, one could still expressly assemble the
training data to have pre-specified class sizes and then tune~$\theta$ for
generalized weighted majority voting \eqref{eq:decision-rule-general}. In our
experiments, we use the usual weighted majority voting
\eqref{eq:decision-rule-with-min} (i.e., $\theta=1$) to classify time series,
where $\Delta_{\max}$ is set to the maximum possible (we consider all shifts).

Per topic, we created its time series based on a pre-processed version of the
raw rate of how often the topic was shared, i.e., its {\em Tweet rate}. We
empirically found that how news topics become trends tends to follow a finite
number of patterns; a few examples of these patterns are shown in Figure
\ref{fig:clusters}. We randomly divided the set of trends and non-trends into
into two halves, one to use as training data and one to use as test data. We
applied weighted majority voting, sweeping over $\gamma$, $T$, and data
pre-processing parameters. As shown in
Figure~\ref{fig:twitter-main}\subref{fig:early}, one choice of parameters
allows us to detect trending topics in advance of Twitter 79\% of the time,
and when we do, we detect them an average of 1.43 hours earlier. Furthermore,
we achieve a true positive rate (TPR) of 95\% and a false positive rate (FPR)
of 4\%. Naturally, there are tradeoffs between TPR, FPR, and how early we make
a prediction (i.e., how small $T$ is). As shown in
Figure~\ref{fig:twitter-main}\subref{fig:roc-early-late}, an ``aggressive''
parameter setting yields early detection and high TPR but high FPR, and a
``conservative'' parameter setting yields low FPR but late detection and low
TPR. An ``in-between'' setting can strike the right balance.

\textbf{Acknowledgements.}
This work was supported in part by the Army Research Office under MURI Award
58153-MA-MUR. GHC was supported by an NDSEG fellowship.

%


\newpage

\small
\bibliography{twitter}

\normalsize

\newpage
\appendix

\section{Proof of Theorem \ref{thm:wmv-main-result}}
\label{sec:proof}

Let $S$ be the time series with an unknown label that we wish to classify
using training data.
Denote $m_+ \triangleq |\cal V_+|$,
$m_- \triangleq |{\cal V}_- |= m - m_+$,
$n_+ \triangleq |{\cal R}_+|$,
$n_- \triangleq |{\cal R}_-|$, and
$\mathcal{R} \triangleq \mathcal{R}_+ \cup \mathcal{R}_-$.
Recall that
${\cal D}_+\triangleq\{0,1,\dots,\Delta_{\max}\}$, and
${\cal D}\triangleq\{-\Delta_{\max},\dots,-1,0,1,\dots,\Delta_{\max}\}$.

As per the model, there exists a latent source $V$, shift
$\Delta' \in {\cal D}_+$, and noise signal $E'$ such that
\begin{equation}
S = V*\Delta' + E'.
\end{equation}
Applying a standard coupon collector's problem result, with a training set of
size $n> \beta m \log m$, then with probability at least $1-m^{-\beta+1}$, for
each latent source $V \in {\cal V}$, there exists at least one time series $R$
in the set ${\cal R}$ of all training data that is generated from $V$.
Henceforth, we assume that this event holds. In Appendix
\ref{sec:non-uniform-latent-sources}, we elaborate on what happens if the
latent sources are not uniformly sampled.

Note that $R$ is generated from $V$ as
\begin{align}
R & = V * \Delta'' + E'',
\end{align}
where $\Delta''\in {\cal D}_+$ and $E''$ is a noise signal independent
of~$E'$. Therefore, we can rewrite $S$ in terms of $R$ as follows:
\begin{align}\label{eq:prf1}
S & = R * \Delta + E,
\end{align}
where $\Delta = \Delta' - \Delta{''} \in {\cal D}$ (note the change from
${\cal D}_+$ to ${\cal D}$) and
$E = E' - E{''}*\Delta$. Since $E'$ and $E{''}$ are i.i.d.~over time and
sub-Gaussian with parameter $\sigma$, one can easily verify that $E$ is
i.i.d.~over time  and sub-Gaussian with parameter $\sqrt{2}\sigma$. 

We now bound the probability of error of classifier
$\widehat{L}_{\theta}^{(T)}(\cdot;\gamma)$.
The probability of error or misclassification using the first $T$ time steps
of $S$ is given by
\begin{align}
 &\mathbb{P}\big(\text{misclassify~}S\text{~using its first~}T
                 \text{~time steps}\big) \nonumber \\
 &=\mathbb{P}(\widehat{L}_\theta^{(T)}(S;\gamma)=-1|L=+1)
   \underbrace{ \mathbb{P}(L=+1) }_{m_+/m}
   \,+\,
   \mathbb{P}(\widehat{L}_\theta^{(T)}(S;\gamma)=+1|L=-1)
   \underbrace{ \mathbb{P}(L=-1) }_{m_-/m}.
\label{eq:pf-prob-error}
\end{align}
In the remainder of the proof, we primarily show how to bound 
$\mathbb{P}(\widehat{L}_\theta^{(T)}(S;\gamma)=-1|L=+1)$. The bound for
$\mathbb{P}(\widehat{L}_\theta^{(T)}(S;\gamma)=+1|L=-1)$ is almost
identical. By Markov's inequality,
\begin{equation}
\mathbb{P}(\widehat{L}_\theta^{(T)}(S;\gamma)=-1|L=+1)
=\mathbb{P}\bigg(
             \frac{1}{\Lambda^{(T)}(S;\gamma)}
               \geq \frac{1}{\theta} \Big| L = +1
           \bigg)
\le
 \theta
 \mathbb{E}\bigg[\frac{1}{\Lambda^{(T)}(S;\gamma)} \Big| L=+1\bigg].
\label{eq:pf-markov-ineq}
\end{equation}
Now, 
\begin{equation}
\mathbb{E}\bigg[\frac{1}{\Lambda^{(T)}(S;\gamma)} \Big| L=+1\bigg]
\le
\max_{r_+\in\mathcal{R}_+, \Delta_+ \in {\cal D}}
  \mathbb{E}_E\bigg[
                \frac{1}{\Lambda^{(T)}(r_+ * \Delta_+ + E; \gamma)}
              \bigg].
\label{eq:pf-helper-bound1}
\end{equation}
With the above inequality in mind, we next bound
$1/\Lambda^{(T)}(\widetilde{r}_+ * \widetilde{\Delta}_+ + E;\gamma)$
for any choice of $\widetilde{r}_+ \in {\cal R}_+$ and
$\widetilde{\Delta}_+ \in {\cal D}$. Note that for any time series $s$,
\begin{equation}
\frac{1}{\Lambda^{(T)}(s;\gamma)}
\le
\frac{\sum_{\substack{r_-\in {\cal R}_-, \\ \Delta_- \in {\cal D}}}
        \exp\big(-\gamma  \|r_-*\Delta_- - s\|_T^2\big)}
     {\exp\big(
            -\gamma \|\widetilde{r}_+ * \widetilde{\Delta}_+ - s\|_T^2
          \big)}.
\end{equation}
After evaluating the above for
$s = \widetilde{r}_+ * \widetilde{\Delta}_+ + E$, a bit of algebra shows that
\begin{align}
&\frac{1}{\Lambda^{(T)}(\widetilde{r}_+ * \widetilde{\Delta}_+ + E;\gamma)}
  \nonumber \\
&\le
  \sum_{\substack{r_-\in {\cal R}_-, \\ \Delta_- \in {\cal D}}}
  \big\{
   \exp\big(-\gamma
             \|\widetilde{r}_+ *\widetilde{\Delta}_+
               - r_{-}*\Delta_{-}\|_T^2\big)
    \exp\big(-2\gamma
              \langle \widetilde{r}_+*\widetilde{\Delta}_+ 
                - r_{-} * \Delta_{-}, E\rangle_T\big)
  \big\},
\label{eq:pf-helper-bound2}
\end{align}
where $\langle q,q'\rangle_T \triangleq \sum_{t=1}^T q(t) q'(t)$ for time
series $q$ and $q'$.

Taking the expectation of \eqref{eq:pf-helper-bound2} with respect to noise
signal $E$, we obtain the following bound: 
\begin{align}
&\mathbb{E}_{E}
 \bigg[
   \frac{1}{\Lambda^{(T)}(\widetilde{r}_+ * \widetilde{\Delta}_+ + E;\gamma)}
 \bigg] \nonumber \\
&\le
  \mathbb{E}_{E}
  \bigg[
    \sum_{\substack{r_-\in {\cal R}_-, \\ \Delta_- \in {\cal D}}}\!\!
   \Big\{
      \exp\big(-\gamma
                \|\widetilde{r}_+ * \widetilde{\Delta}_+
                  - r_{-} * \Delta_{-}\|^2_T\big)
     \exp\big(-2\gamma
               \langle \widetilde{r}_+ * \widetilde{\Delta}_+ - r_{-}*\Delta_{-}, E\rangle_T\big) 
     \Big\}
  \bigg]
  \nonumber \\
&\overset{(i)}{=}\!
    \sum_{\substack{r_{-}\in\mathcal{R}_-, \\
                    \Delta_{-}\in{\cal D}}
         }\!\!
      \exp\big(-\gamma \|\widetilde{r}_+ * \widetilde{\Delta}_+
                         - r_{-} * \Delta_{-}\|^2_T\big)
      \prod_{t=1}^T
        \mathbb{E}_{E(t)}
        [ \exp\big(-2\gamma(\widetilde{r}_+(t+ \widetilde{\Delta}_+)
                            - r_{-}(t+\Delta_{-})) E(t)\big) ]
  \nonumber \\
&\!\overset{(ii)}{\le}\!
 \sum_{\substack{r_{-}\in\mathcal{R}_-, \\
                    \Delta_{-}\in{\cal D}}
         }\!\!
     \exp\big(-\gamma \|\widetilde{r}_+ * \widetilde{\Delta}_+
                        - r_{-} * \Delta_{-}\|^2_T\big) 
  \prod_{t=1}^T
      \exp\big(4\sigma^2\gamma^2
           (\widetilde{r}_+(t + \widetilde{\Delta}_+)
            -
            r_{-}(t + \Delta_{-}))^2\big)
  \nonumber \\
& = \sum_{\substack{r_{-}\in\mathcal{R}_-, \\
                    \Delta_{-}\in{\cal D}}
         }\!\!
    \exp\big(-(\gamma-4\sigma^2 \gamma^2) \|r_+ * \Delta_+ - r_{-} * \Delta_{-}\|^2_T\big) 
  \nonumber \\
&\le (2\Delta_{\max}+1) n_- \exp\big(-(\gamma-4\sigma^2 \gamma^2) G^{(T)}\big),
\label{eq:pf-helper-bound3}
\end{align}
where step $(i)$ uses independence of entries of $E$, step $(ii)$ uses
the fact that $E(t)$ is zero-mean sub-Gaussian with parameter
$\sqrt{2}\sigma$, and the last line abbreviates the gap
$G^{(T)}\equiv G^{(T)}(\mathcal{R}_+,\mathcal{R}_-,\Delta_{\max})$.

Stringing together inequalities \eqref{eq:pf-markov-ineq},
\eqref{eq:pf-helper-bound1}, and \eqref{eq:pf-helper-bound3}, we obtain 
\begin{equation}
\mathbb{P}(\widehat{L}_\theta^{(T)}(S;\gamma)=-1|L=+1)
\le
 \theta
 (2\Delta_{\max}+1) n_- \exp\big(-(\gamma-4\sigma^2 \gamma^2) G^{(T)}\big).
\label{eq:pf-misclassify-bound++}
\end{equation}
Repeating a similar argument yields
\begin{equation}
\mathbb{P}(\widehat{L}_\theta^{(T)}(S;\gamma)=+1|L=-1)
\le
 \frac{1}{\theta}
 (2\Delta_{\max}+1) n_+ \exp\big(-(\gamma-4\sigma^2 \gamma^2) G^{(T)}\big).
\label{eq:pf-misclassify-bound--}
\end{equation}
Finally, plugging \eqref{eq:pf-misclassify-bound++} and
\eqref{eq:pf-misclassify-bound--} into \eqref{eq:pf-prob-error} gives
\begin{align}
\mathbb{P}(\widehat{L}_{\theta}^{(T)}(S;\gamma)\ne L)
&\le
  \theta (2\Delta_{\max}+1)
  \frac{n_- m_+}{m}
  \exp\big(-(\gamma-4\sigma^2 \gamma^2) G^{(T)}\big) \nonumber \\
&\quad
  +
  \frac{1}{\theta} (2\Delta_{\max}+1)
  \frac{n_+ m_-}{m}
  (2\Delta_{\max}+1) n_+ \exp\big(-(\gamma-4\sigma^2 \gamma^2) G^{(T)}\big)
 \nonumber \\
&=
  \Big( \frac{\theta m_+}{m} + \frac{m_-}{\theta m} \Big)
  (2\Delta_{\max}+1) n
  \exp\big(-(\gamma-4\sigma^2 \gamma^2) G^{(T)}\big). 
\end{align}
This completes the proof
of Theorem \ref{thm:wmv-main-result}.


\section{Proof of Theorem \ref{thm:nn-main-result}}
\label{sec:proof-nn}
 
The proof uses similar steps as the weighted majority voting case. As before,
we consider the case when our training data sees each latent source at least
once (this event happens with probability at least $1-m^{-\beta+1}$).

We decompose the probability of error into terms depending on which latent
source $V$ generated $S$:
\begin{equation}
\mathbb{P}(\widehat{L}_{NN}^{(T)}(S)\ne L)
=\sum_{v\in\mathcal{V}}
   \mathbb{P}(V=v)
   \mathbb{P}(\widehat{L}_{NN}^{(T)}(S)\ne L|V=v)
=\sum_{v\in\mathcal{V}}
   \frac{1}{m}
   \mathbb{P}(\widehat{L}_{NN}^{(T)}(S)\ne L|V=v).
\label{eq:nn-prob-error}
\end{equation}
Next, we bound each $\mathbb{P}(\widehat{L}_{NN}^{(T)}(S)\ne L|V=v)$ term.
Suppose that $v\in\mathcal{V}_+$, i.e., $v$ has label $L=+1$; the case when
$v\in\mathcal{V}_-$ is similar. Then we make an error and declare
$\widehat{L}_{NN}^{(T)}(S)=-1$ when the nearest neighbor $\widehat{r}$ to time
series $S$ is in the set $\mathcal{R}_-$, where
\begin{equation}
(\widehat{r},\widehat{\Delta})
=\arg\min_{(r,\Delta)\in\mathcal{R}\times\mathcal{D}}
   \|r*\Delta-S\|_T^2.
\label{eq:nn-opt}
\end{equation}
By our assumption that every latent source is seen in the training data, there
exists $r^{*}\in\mathcal{R}_{+}$ generated by latent source $v$, and so
\begin{equation}
S=r^* * \Delta^*+E
\end{equation}
for some shift $\Delta^*\in\mathcal{D}$ and noise signal $E$ consisting of
i.i.d.~entries that are zero-mean sub-Gaussian with parameter
$\sqrt{2}\sigma$.

By optimality of $(\widehat{r},\widehat{\Delta})$ for optimization problem
\eqref{eq:nn-opt}, we have
\begin{equation}
\|r*\Delta - (r^* * \Delta^*+E)\|_T^2
\ge
  \|\widehat{r} * \widehat{\Delta} - (r^* * \Delta^* + E)\|_T^2
\qquad
\text{for all }r\in\mathcal{R},\Delta\in\mathcal{D}.
\end{equation}
Plugging in $r=r^*$ and $\Delta=\Delta^*$, we obtain
\begin{align}
\|E\|_T^2
&
\ge\|\widehat{r} * \widehat{\Delta} - (r^* * \Delta^* + E)\|_T^2
 \nonumber \\
&
=\|(\widehat{r} * \widehat{\Delta} - r^* * \Delta^*) - E\|_T^2
 \nonumber \\
&
=\|\widehat{r} * \widehat{\Delta} - r^* * \Delta^*\|_T^2
 - 2\langle\widehat{r} * \widehat{\Delta} - r^* * \Delta^*, E \rangle_T
 + \|E\|_T^2,
\end{align}
or, equivalently,
\begin{equation}
2\langle
   \widehat{r} * \widehat{\Delta} - r^* * \Delta^*, E
 \rangle_T
\ge
 \|\widehat{r} * \widehat{\Delta} -r^* * \Delta^*\|_T^2.
\label{eq:nn-opt-condition}
\end{equation}
Thus, given $V=v\in\mathcal{V}_{+}$, declaring $\widehat{L}_{NN}^{(T)}(S)=-1$
implies the existence of $\widehat{r}\in\mathcal{R}_-$ and
$\widehat{\Delta}\in\mathcal{D}$ such that optimality condition
\eqref{eq:nn-opt-condition} holds. Therefore,
\begin{align}
&\mathbb{P}(\widehat{L}_{NN}^{(T)}(S)=-1|V=v) \nonumber \\
&\le
  \mathbb{P}
  \bigg(
    \bigcup_{\widehat{r}\in\mathcal{R}_-,
             \widehat{\Delta}\in\mathcal{D}}
    \{
      2\langle
         \widehat{r} * \widehat{\Delta} - r^* * \Delta^*, E
       \rangle_T
      \ge
       \|\widehat{r} * \widehat{\Delta} -r^* * \Delta^*\|_T^2
    \}
  \bigg)
  \nonumber \\
&\overset{(i)}{\le}\!
  (2\Delta_{\max}+1)
  n_-
  \mathbb{P}
  (
    2\langle
       \widehat{r} * \widehat{\Delta} - r^* * \Delta^*, E
     \rangle_T
    \ge
     \|\widehat{r} * \widehat{\Delta} -r^* * \Delta^*\|_T^2
  )
  \nonumber \\
&\le
  (2\Delta_{\max}+1)
  n_-
  \mathbb{P}
  (
    \exp(2\lambda
         \langle
           \widehat{r} * \widehat{\Delta} - r^* * \Delta^*, E
         \rangle_T
        )
    \ge
    \exp(\lambda
         \|\widehat{r} * \widehat{\Delta} -r^* * \Delta^*\|_T^2
        )
  )
  \nonumber \\
&\!\overset{(ii)}{\le}\!
  (2\Delta_{\max}+1)
  n_-
  \exp(-\lambda
        \|\widehat{r} * \widehat{\Delta} -r^* * \Delta^*\|_T^2
      )
  \mathbb{E}
  [
    \exp(2\lambda
         \langle
           \widehat{r} * \widehat{\Delta} - r^* * \Delta^*, E
         \rangle_T
        )
  ]
  \nonumber \\
&\!\!\overset{(iii)}{\le}\!\!
  (2\Delta_{\max}+1)
  n_-
  \exp(-\lambda
        \|\widehat{r} * \widehat{\Delta} -r^* * \Delta^*\|_T^2
      )
  \exp(4\lambda^2\sigma^2
       \|\widehat{r} * \widehat{\Delta} - r^* * \Delta^*\|_T^2
      )
  \nonumber \\
&=
  (2\Delta_{\max}+1)
  n_-
  \exp(-(\lambda-4\lambda^2\sigma^2)
        \|\widehat{r} * \widehat{\Delta} -r^* * \Delta^*\|_T^2
      )
  \nonumber \\
&\le
 (2\Delta_{\max}+1)
 n
 \exp(-(\lambda-4\lambda^2\sigma^2)G^{(T)})
 \nonumber \\
&\!\!\overset{(iv)}{\le}\!
 (2\Delta_{\max}+1) n
 \exp\Big( -\frac{1}{16\sigma^2}G^{(T)} \Big),
\label{eq:nn-main-bound-reproduced}
\end{align}
where step $(i)$ is by a union bound, step $(ii)$ is by Markov's inequality,
step $(iii)$ is by sub-Gaussianity, and step $(iv)$ is by choosing
$\lambda=\frac{1}{8\sigma^2}$.

As bound \eqref{eq:nn-main-bound-reproduced} also holds for
$\mathbb{P}(\widehat{L}_{NN}^{(T)}(S)=+1|V=v)$ when instead
$v\in\mathcal{V}_-$, we can now piece together \eqref{eq:nn-prob-error}
and \eqref{eq:nn-main-bound-reproduced} to yield the final result:
\begin{equation}
\mathbb{P}(\widehat{L}_{NN}^{(T)}(S)\ne L)
=\sum_{v\in\mathcal{V}}
   \frac{1}{m}
   \mathbb{P}(\widehat{L}_{NN}^{(T)}(S)\ne L|V=v)
\le
 (2\Delta_{\max}+1) n
 \exp\Big(-\frac{1}{16\sigma^{2}}G^{(T)}\Big).
\end{equation}

\section{Handling Non-uniformly Sampled Latent Sources}
\label{sec:non-uniform-latent-sources}

When each time series generated from the latent source model is sampled
uniformly at random, then having $n>m\log\frac{2m}{\delta}$ (i.e.,
$\beta=1+\log\frac{2}{\delta}/\log m$) ensures that with probability at least
$1-\frac{\delta}{2}$, our training data sees every latent source at least
once. When the latent sources aren't sampled uniformly at random, we show that
we can simply replace the condition $n>m\log\frac{2m}{\delta}$ with
$n\ge\frac{8}{\pi_{\min}}\log\frac{2m}{\delta}$ to achieve a similar (in fact,
stronger) guarantee, where $\pi_{\min}$ is the smallest probability
of a particular latent source occurring.

\begin{lemma}
Suppose that the $i$-th latent source occurs with probability $\pi_i$ in the
latent source model. Denote
$\pi_{\min}
 \triangleq
   \min_{i\in\{1,2,\dots,m\}} \pi_i$.
Let $\xi_i$ be the number of times that the $i$-th latent source appears in
the training data. If $n\ge\frac{8}{\pi_{\min}}\log\frac{2m}{\delta}$, then
with probability at least $1-\frac{\delta}{2}$, every latent source appears
strictly greater than $\frac{1}{2}n\pi_{\min}$ times in the training data.
\end{lemma}
\begin{proof}
Note that $\xi_i\sim\text{Bin}(n,\pi_i)$. We have
\begin{align}
\mathbb{P}\big(\xi_i \le \frac{1}{2}n\pi_{\min}\big)
&\le
  \mathbb{P}\big(\xi_i\le\frac{1}{2}n\pi_i\big)
  \nonumber \\
&\overset{(i)}{\le}\!
  \exp\Big(
        -\frac{1}{2}
         \cdot
         \frac{(n\pi_i-\frac{1}{2}n\pi_i)^2}
              {n\cdot\pi_i}
      \Big)
  \nonumber \\
&=\exp\Big( -\frac{n\pi_i}{8} \Big)
  \nonumber \\
&\le
  \exp\Big( -\frac{n\pi_{\min}}{8} \Big).
\end{align}
where step $(i)$ uses a standard binomial distribution lower tail bound.
Applying a union bound,
\begin{equation}
\mathbb{P}\bigg(
            \bigcup_{i\in\{1,2,\dots,m\}}
            \big\{
              \xi_i\le\frac{1}{2}n\pi_{\min}
            \big\}
          \bigg)
\le m\exp\Big(-\frac{n\pi_{\min}}{8}\Big),
\end{equation}
which is at most $\frac{\delta}{2}$ when
$n\ge\frac{8}{\pi_{\min}}\log\frac{2m}{\delta}$.
\end{proof}

\section{Sample Complexity for the Gaussian Setting Without Time Shifts}
\label{sec:gaussian-classification}

Existing results on learning mixtures of Gaussians by Dasgupta and Schulman
\cite{dasgupta_2007} and by Vempala and Wang \cite{vempala_wang} use a
different notion of gap than we do. In our notation, their gap can be written
as
\begin{equation}
G^{(T)*}
\triangleq
  \min_{v, v'\in\mathcal{V}\text{ s.t.\,}v\ne v'}\|v-v'\|_T^2,
\end{equation}
which measures the minimum separation between the true latent sources,
disregarding their labels.

We now translate our main theoretical guarantees to be in terms of gap
$G^{(T)*}$ under the assumption that the noise is Gaussian and that there are
no time shifts.

\begin{theorem}
Under the latent source model, suppose that the noise is zero-mean Gaussian
with variance $\sigma^{2}$, that there are no time shifts (i.e.,
$\Delta_{\max}=0$), and that we have sampled $n>m\log\frac{4m}{\delta}$
training time series. Then if
\begin{align}
G^{(T)*} & \ge 4\sigma^2 \log \frac{4n^2}{\delta},\\
       T & \ge (12+8\sqrt{2}) \log \frac{4n^2}{\delta},
\end{align}
then weighted majority voting (with $\theta=1, \gamma=\frac{1}{8\sigma^{2}}$)
and nearest-neighbor classification each classify a new time series
correctly with probability at least $1-\delta$.
\end{theorem}
In particular, with access to a pool of $\Omega(m\log\frac{m}{\delta})$
time series, we can subsample $n=\Theta(m\log\frac{m}{\delta})$ of
them to use as training data. Then provided that $G^{(T)*}=\Omega(\sigma^{2}\log\frac{m}{\delta})$
and $T=\Omega(\log\frac{m}{\delta})$, we correctly classify a new
time series with probability at least $1-\delta$.

\begin{proof}
The basic idea is to show that with high probability, our gap
$G^{(T)} \equiv G^{(T)}(\mathcal{R}_+, \mathcal{R}_-, \Delta_{\max})$
satisfies
\begin{equation}
G^{(T)}
\ge G^{(T)*} + 2\sigma^2 T
    - 4\sigma\sqrt{G^{(T)*} \log\frac{4n^2}{\delta}}
    - 4\sigma^2 \sqrt{T\log\frac{4n^2}{\delta}}.
\end{equation}
The worst-case scenario occurs when
$G^{(T)*} = 4\sigma^2 \log \frac{4n^2}{\delta}$, at which point we have
\begin{equation}
G^{(T)}
\ge 2\sigma^2 T
    - 4\sigma^2 \log\frac{4n^2}{\delta}
    - 4\sigma^2 \sqrt{T\log\frac{4n^2}{\delta}}.
\end{equation}
The right-hand side is at least $\sigma^{2}T$ when
\begin{equation}
T \ge (12+8\sqrt{2}) \log\frac{4n^2}{\delta},
\end{equation}
which ensures that, with high probability, $G^{(T)} \ge \sigma^2 T$.
Theorems \ref{thm:wmv-main-result} and \ref{thm:nn-main-result} each say that
if we further have $n > m\log\frac{4m}{\delta}$, and
$T \ge 16\log\frac{4n}{\delta}$, then we classify a new time series
correctly with high probability, where we note that 
$T \ge (12+8\sqrt{2})\log\frac{4n^2}{\delta}
   \ge 16\log\frac{4n}{\delta}$.

We now fill in the details. Let $r_+$ and $r_-$ be two time series in the
training data that have labels $+1$ and $-1$ respectively, where we assume
that $\Delta_{\max}=0$. Let $v^{(r_+)}\in\mathcal{V}$ and
$v^{(r_-)}\in\mathcal{V}$ be the true latent sources of $r_+$ and $r_-$,
respectively. This means that
$r_+ \sim \mathcal{N}(v^{(r_+)}, \sigma^2 I_{T\times T})$ and
$r_- \sim \mathcal{N}(v^{(r_-)}, \sigma^2 I_{T\times T})$. Denoting
$E^{(r_+)} \sim \mathcal{N}(0, \sigma^2 I_{T\times T})$ and
$E^{(r_-)} \sim \mathcal{N}(0, \sigma^2 I_{T\times T})$ to be noise associated
with time series $r_+$ and $r_-$, we have
\begin{align}
\|r_+ - r_-\|_T^2
&= \| (v^{(r_+)} + E^{(r_+)}) - (v^{(r_-)} + E^{(r_-)}) \|_T^2 \nonumber \\
&= \| v^{(r_+)} - v^{(r_-)} \|_T^2
   +
   2\langle
      v^{(r_+)} - v^{(r_-)},
      E^{(r_+)} - E^{(r_-)}
    \rangle
   +
   \| E^{(r_+)} - E^{(r_-)} \|_T^2.
\end{align}
We shall show that with high probability, for all $r_+ \in \mathcal{R}_+$ and
for all $r_- \in \mathcal{R}_-$:
\begin{align}
\langle
  v^{(r_+)} - v^{(r_-)},
  E^{(r_+)} - E^{(r_-)}
\rangle
& \ge -2 \sigma\| v^{(r_+)} - v^{(r_-)} \|_T
       \sqrt{\log\frac{4 n^2}{\delta}},
  \label{eq:gap-true-sources-thing1} \\
\| E^{(r_+)} - E^{(r_-)} \|_T^2
& \ge 2\sigma^2 T - 4\sigma^2 \sqrt{T\log\frac{4 n^2}{\delta}}.
  \label{eq:gap-true-sources-thing2}
\end{align}

\begin{itemize}
\item Bound \eqref{eq:gap-true-sources-thing1}:
  $\langle v^{(r_+)} - v^{(r_-)}, E^{(r_+)} - E^{(r_-)} \rangle$ is zero-mean
  sub-Gaussian with parameter $\sqrt{2}\sigma\| v^{(r_+)} - v^{(r_-)} \|_T$,
  so
  \begin{equation}
    \mathbb{P}( \langle
                  v^{(r_+)} - v^{(r_-)},
                  E^{(r_+)} - E^{(r_-)}
                \rangle
                  \le -a)
    \le
      \exp\Big(
            -\frac{a^2}{4 \sigma^2 \| v^{(r_+)} - v^{(r_-)} \|_T^2}
          \Big)
    = \frac{\delta}{4n^2}
  \end{equation}
  with choice
  $a=2 \sigma \| v^{(r_+)} - v^{(r_-)} \|_T
     \sqrt{\log \frac{4n^2}{\delta} }$.
  A union bound over all pairs of time series in the training data with
  opposite labels gives
  \begin{equation}
  \mathbb{P}\Bigg(
              \bigcup_{\substack{r_+ \in \mathcal{R}_+, \\
                                 r_- \in \mathcal{R}_-}}
                \bigg\{
                  \langle
                    v^{(r_+)} - v^{(r_-)},
                    E^{(r_+)} - E^{(r_-)}
                  \rangle
                  \le
                    -2 \sigma\| v^{(r_+)} - v^{(r_-)} \|_T
                     \sqrt{\log \frac{4n^2}{\delta} }
                \bigg\}
            \Bigg)
            \le
              \frac{\delta}{4}.
  \end{equation}

\item Bound \eqref{eq:gap-true-sources-thing2}:
  Due to a result by Laurent and Massart \cite[Lemma 1]{chi_square_tails},
  we have
  \begin{equation}
  \mathbb{P}( \| E^{(r_+)} - E^{(r_-)} \|_T^2
              \le 2\sigma^2 T
                  - 4\sigma^2 \sqrt{Ta})
    \le e^{-a}
    = \frac{\delta}{4n^2}
  \end{equation}
  with choice $a=\log\frac{4n^2}{\delta}$. A union bound gives
  \begin{equation}
  \mathbb{P}\Bigg(
              \bigcup_{\substack{r_+\in\mathcal{R}_+, \\
                                 r_-\in\mathcal{R}_-}}
                \bigg\{
                  \| E^{(r_+)} - E^{(r_-)} \|_T^2
                  \le
                    2\sigma^2 T
                    - 4\sigma^2\sqrt{T\log\frac{4n^2}{\delta}}
                \bigg\}
            \Bigg)
    \le \frac{\delta}{4}.
  \end{equation}

\end{itemize}

Assuming that bounds \eqref{eq:gap-true-sources-thing1} and
\eqref{eq:gap-true-sources-thing2} both hold, then for all
$r_+\in\mathcal{R}_{+},r_-\in\mathcal{R}_{-}$, we have
\begin{align}
&\| r_+ - r_- \|_T^2 \nonumber \\
&=\| v^{(r_+)} - v^{(r_-)} \|_T^2
  + 2\langle
       v^{(r_+)} - v^{(r_-)},
       E^{(r_+)} - E^{(r_-)}
     \rangle
  + \| E^{(r_+)} - E^{(r_-)} \|_T^2 \nonumber \\
&\ge
  \| v^{(r_+)} - v^{(r_-)} \|_T^2
  - 4\sigma \| v^{(r_+)} - v^{(r_-)} \|_T
     \sqrt{\log \frac{4n^2}{\delta} }
  + 2\sigma^2 T
  - 4\sigma^2 \sqrt{T\log\frac{4n^2}{\delta}} \nonumber \\
&\overset{(i)}{=}
  \bigg(
    \| v^{(r_+)} - v^{(r_-)} \|_T
    - 2\sigma\sqrt{\log \frac{4n^2}{\delta} }
  \bigg)^2
  - 4\sigma^2
     \log\frac{4n^2}{\delta}
  + 2\sigma^2 T
  - 4\sigma^2 \sqrt{T\log\frac{4n^2}{\delta}} \nonumber \\
&\overset{(ii)}{\ge}
  \bigg(
    \sqrt{G^{(T)*}}
    - 2\sigma\sqrt{\log \frac{4n^2}{\delta} }
  \bigg)^2
  + 2\sigma^2 T
  - 4\sigma^2 \log\frac{4n^2}{\delta}
  - 4\sigma^2 \sqrt{T\log\frac{4n^2}{\delta}},
\end{align}
where step $(i)$ follows from completing the square, and step $(ii)$ uses our
assumption that $G^{(T)*} \ge 4\sigma^2 \log\frac{4n^2}{\delta}$.
Minimizing over $r_+ \in \mathcal{R}_+$ and $r_- \in \mathcal{R}_-$, we get
\begin{equation}
G^{(T)}
\ge
  \bigg(
    \sqrt{G^{(T)*}}
    - 2\sigma\sqrt{\log\frac{4n^2}{\delta}}
  \bigg)^2
  + 2\sigma^2 T
  - 4\sigma^2 \log\frac{4n^2}{\delta}
  - 4\sigma^2 \sqrt{T\log\frac{4n^2}{\delta}}.
\end{equation}
The worst-case scenario occurs when
$G^{(T)*} = 4\sigma^2 \log\frac{4n^2}{\delta}$, in which case
\begin{equation}
G^{(T)}
\ge 2\sigma^2 T
    - 4\sigma^2 \log\frac{4n^2}{\delta}
    - 4\sigma^2 \sqrt{T\log\frac{4n^2}{\delta}}.
\end{equation}
Theorems \ref{thm:wmv-main-result} and \ref{thm:nn-main-result} imply that
having $G^{(T)} \ge \sigma^2 T$, $n > m\log\frac{4m}{\delta}$, and
$T \ge 16\log\frac{4n}{\delta}$ allows weighted majority voting
(with $\theta=1, \gamma=\frac{1}{8\sigma^2}$) and nearest-neighbor
classification to each succeed with high probability. We achieve
$G^{(T)} \ge \sigma^2 T$ by asking that
\begin{equation}
2\sigma^2 T
- 4\sigma^2 \log\frac{4n^2}{\delta}
- 4\sigma^2 \sqrt{T\log\frac{4n^2}{\delta}}
\ge \sigma^2 T,
\end{equation}
which happens when
\begin{equation}
T \ge (12 + 8\sqrt{2})\log\frac{4n^2}{\delta}.
\end{equation}
A union bound over the following four bad events (each controlled to happen
with probability at most $\frac{\delta}{4}$) yields the final result:
\begin{itemize}
\item Not every latent source is seen in the training data.
\item Bound \eqref{eq:gap-true-sources-thing1} doesn't hold.
\item Bound \eqref{eq:gap-true-sources-thing2} doesn't hold.
\item Assuming that the above three bad events don't happen, we still misclassify.
      $\hfill\qedhere$
\end{itemize}
\end{proof}

\section{Forecasting Trending Topics on Twitter}
\label{sec:forecasting-trends}

Twitter is a social network whose users post
messages called {\em Tweets}, which are then
broadcast to a user's followers. Often, emerging topics of interest are
discussed on Twitter in real time. Inevitably, certain topics gain sudden
popularity and --- in Twitter speak --- begin to {\em trend}.
Twitter surfaces
such topics as a list of top ten {\em trending topics}, or {\em trends}.

\textbf{Data.} We sampled 500 examples of trends at random from a list of June
2012 news trends and recorded the earliest time each topic trended within the
month. Before sampling, we filtered out trends that never achieved a rank of 3
or better on the Twitter trends list%
\footnote{On Twitter, trending topics compete for the top ten spots whereas we
are only detecting whether a topic will trend or not.}
as well as trends that lasted for less than 30
minutes as to keep our trend examples reasonably salient. We also sampled 500
examples of non-trends at random from a list of $n$-grams (of sizes 1, 2, and
3) appearing in Tweets created in June 2012, where we filter out any $n$-gram
containing words that appeared in one of our 500 chosen trend examples. Note
that as we do not know how Twitter chooses what phrases are considered as
topic phrases (and are candidates for trending topics), it's unclear what the
size of the non-trend category is in comparison to the size of the
trend category. Thus, for simplicity, we intentionally control for
the class sizes by setting them equal. In practice, one could still
expressly assemble the training data to have pre-specified class sizes and
then tune~$\theta$ for generalized weighted majority voting
\eqref{eq:decision-rule-general}. In our experiments, we just use the usual
weighted majority voting \eqref{eq:decision-rule-with-min} (i.e., $\theta=1$)
to classify time series.

From these examples of trends and non-trends, we then created time series of
activity for each topic based on the rate of Tweets about that topic over
time. To approximate this rate, we gathered 10\% of all Tweets from June 2012,
placed them into two-minute buckets according to their timestamps, and counted
the number of Tweets in each bucket. We denote the count at the $t$-th time
bucket as $\rho(t)$, which we refer to as the raw rate. We then transform the
raw rate in a number of ways, summarized in
Figure~\ref{fig:twitter-preprocessing}, before using the resulting time series
for classification.

\begin{figure}[t]
\centering
\includegraphics[width=\linewidth, clip=true, trim=0em .4in 0em 0em]{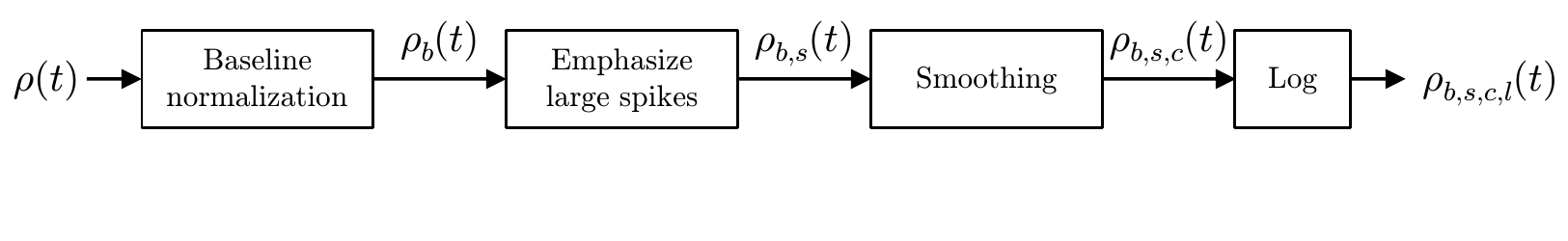}
\caption{Twitter data pre-processing pipeline:
  The raw rate $\rho(t)$ counts the number of Tweets in time bucket $t$.
  We normalize $\rho(t)$ to make the counts relative:
  $\rho_b(t) \triangleq \rho(t) / \sum_{\tau=1}^t \rho(\tau)$.
  Large spikes are emphasized: 
  $\rho_{b,s}(t) \triangleq | \rho_b(t) - \rho_b(t-1) |^{\alpha}$
  (we use $\alpha=1.2$).
  Next, we smooth the signal:
  $\rho_{b,s,c}(t)
   \triangleq \sum_{\tau = t - T_{smooth} + 1}^t \rho_{b,s}(\tau)$.
  Finally, we take the log:
  $\rho_{b,s,c,l}(t) \triangleq \log \rho_{b,s,c}(t)$.}
\label{fig:twitter-preprocessing}
\end{figure}

We observed
that trending activity is characterized by spikes above some baseline rate,
whereas non-trending activity has fewer, if any spikes. For example, a
non-trending topic such as ``city'' has a very high, but mostly constant rate
because it is a common word. In contrast, soon-to-be-trending topics
like ``Miss USA'' will initially have a low rate, but will also have bursts in
activity as the news spreads. To emphasize the parts of the rate signal above
the baseline and de-emphasize the parts below the baseline, we define a
baseline-normalized signal
$\rho_b(t) \triangleq \rho(t) / \sum_{\tau=1}^t \rho(\tau)$.

A related observation is that the Tweet rate for a trending topic typically
contains larger and more sudden spikes than those of non-trending topics. We
reward such spikes by emphasizing them, while de-emphasizing smaller spikes.
To do so, we define a baseline-and-spike-normalized rate
$\rho_{b,s}(t) \triangleq | \rho_b(t) - \rho_b(t-1) |^{\alpha}$
in terms of the already baseline-normalized rate $\rho_b$; parameter
$\alpha \geq 1$ controls how much spikes are rewarded (we used $\alpha =
1.2$). In addition, we convolve the result with a smoothing window to
eliminate noise and effectively measure the volume of Tweets in a sliding
window of length
$T_{smooth}$:
$
\rho_{b,s,c}(t)
\triangleq \sum_{\tau = t - T_{smooth} + 1}^t \rho_{b,s}(\tau)
$.

Finally, the spread of a topic from person to person can be thought of as a
branching process in which a population of users ``affected'' by a topic grows
exponentially with time, with the exponent depending on the details of the
model \cite{Asur}. This intuition suggests using a logarithmic scaling for
the volume of Tweets: $\rho_{b,s,c,l}(t) \triangleq \log \rho_{b,s,c}(t)$.

The resulting time series $\rho_{b,s,c,l}$ contains data from the entire
window in which data was collected. To construct the sets of training time
series $\mathcal{R}_+$ and $\mathcal{R}_-$, we keep only a small $h$-hour
slice of representative activity $r$ for each topic. Namely, each of the final
time series $r$ used in the training data is truncated to only contain the $h$
hours of activity in the corresponding transformed time series
$\rho_{b,s,c,l}$. For time series corresponding to trending topics, these $h$
hours are taken from the time leading up to when the topic was first declared
by Twitter to be trending. For time series corresponding to non-trending
topics, the $h$-hour window of activity is sampled at random from all the
activity for the topic. We empirically found that how news topics become
trends tends to follow a finite number of patterns; a few examples of these
patterns are shown in Figure \ref{fig:clusters}.

\textbf{Experiment.} For a fixed choice of parameters, we randomly divided the
set of trends and non-trends into two halves, one for training and one for
testing. Weighted majority voting with the training data was used to classify
the test data. Per time series in the test data, we looked
within a window of $2h$ hours, centered at the trend onset for trends, and
sampled randomly for non-trends. We restrict detection to this time window to
avoid detecting earlier times that a topic became trending, if it trended
multiple times. We then measured the false positive rate (FPR), true positive rate (TPR),
and the time of detection if any. For trends, we computed how early or late
the detection was compared to the true trend onset. We explored the following
parameters: $h$, the length in hours of each example time series; $T$, the
number of initial time steps in the observed time series $s$ that we use for
classification; $\gamma$, the scaling parameter; $T_{smooth}$, the width of
the smoothing window. In all cases, constant $\Delta_{\max}$ in the decision
rule \eqref{eq:decision-rule-with-min} is set to be the maximum possible,
i.e., since observed signal $s$ has $T$ samples, we compare $s$ with all
$T$-sized chunks of each time series $r$ in training data.


For a variety of parameters, we detect 
trending topics before they appear on Twitter's trending topics list. Figure
\ref{fig:twitter-main}
\subref{fig:early} shows that for one such choice of parameters, we 
detect trending topics before Twitter does 79\% of the time, and when we do, we
detect them an average of 1.43 hours earlier. Furthermore, we achieve 
a TPR of 95\% and a FPR of 4\%.
Naturally, there are tradeoffs between the FPR, the TPR, and relative detection
time that depend on parameter settings. An
aggressive parameter setting will yield early detection and a high TPR, but
at the expense of a high FPR. A conservative parameter setting
will yield a low FPR, but at the expense of late detection and a low
TPR. An in-between setting can strike the right balance. We show this
tradeoff in two ways. First, by varying a single parameter at a time and
fixing the rest, we generated an ROC curve that
describes the tradeoff between FPR and TPR. Figure
\ref{fig:twitter-main}
\subref{fig:roc} shows the
envelope of all ROC curves, which can be interpreted as the best ``achievable''
ROC curve.
Second, we broke the results up by where they fall on the ROC curve --- top
(``aggressive''), bottom (``conservative''), and center (``in-between'') --- and
showed the distribution of early and late relative detection times for each
(Figure \ref{fig:twitter-main}\subref{fig:roc-early-late}).

We discuss some fine details of the experimental setup. Due to restrictions on
the Twitter data available, while we could determine whether a trending topic
is categorized as news based on user-curated lists of ``news'' people on
Twitter, we did not have such labels for individual Tweets. Thus, the example
time series that we use as training data contain Tweets that are both news and
non-news. We also reran our experiments using only non-news Tweets and found
similar results except that we do not detect trends as early as before;
however, weighted majority voting still detects trends in advance of Twitter
79\% of the time.

\end{document}